\documentclass{article} 
\usepackage{nips13submit_e,times}
\usepackage{url}
\usepackage{natbib}
\usepackage[utf8]{inputenc}
\usepackage{amsfonts}
\usepackage{amsmath}
\usepackage{amsthm}
\usepackage{algorithm}
\usepackage[noend]{algorithmic}
\usepackage{graphicx}
\usepackage{caption}
\usepackage{sidecap}
\newtheorem{proposition}{Proposition}

\newtheorem{corollary}{Corollary}
\newtheorem{theorem}{Theorem}

\nipsfinaltrue

\def\R{{\mathbf R}}
\newcommand{\vs}[1]{\vspace*{-#1mm}}

\title{Generalized Denoising Auto-Encoders as Generative Models}

\newif\ifLong
\Longfalse
\newif\ifConvergence
\Convergencefalse

\author{
Yoshua Bengio, Li Yao, Guillaume Alain, and Pascal Vincent\\
Département d'informatique et recherche opérationnelle, Université de Montréal
}

\begin{document}

\maketitle
\vs{3}
\begin{abstract}
Recent work has shown how denoising and contractive autoencoders implicitly capture the structure of the data-generating density, in the case where the corruption noise is Gaussian, the reconstruction error is the squared error, and the data is continuous-valued. This has led to various proposals for sampling from this implicitly learned density function, using Langevin and Metropolis-Hastings MCMC. However, it remained unclear how to connect the training procedure of regularized auto-encoders to the implicit estimation of the underlying data-generating distribution when the data are discrete, or using other forms of corruption process and reconstruction errors. Another issue is the mathematical justification which is only valid in the limit of small corruption noise.  We propose here a different attack on the problem, which deals with all these issues: arbitrary (but noisy enough) corruption, arbitrary reconstruction loss (seen as a log-likelihood), handling both discrete and continuous-valued variables, and removing the bias due to non-infinitesimal corruption noise (or non-infinitesimal contractive penalty).
\end{abstract}

\vs{3}
\section{Introduction}
\vs{2}

Auto-encoders learn an encoder function from input to representation and a
decoder function back from representation to input space, such that the
reconstruction (composition of encoder and decoder) is good for training
examples. Regularized auto-encoders also involve some form of
regularization that prevents the auto-encoder from simply learning the
identity function, so that reconstruction error will be low at training
examples (and hopefully at test examples) but high in general.  Different
variants of auto-encoders and sparse coding have been, along with RBMs,
among the most successful building blocks in recent research in deep
learning~\citep{Bengio-Courville-Vincent-TPAMI2013}. Whereas the usefulness
of auto-encoder variants as feature learners for supervised learning can
directly be assessed by performing supervised learning experiments with
unsupervised pre-training, what has remained until recently rather unclear
is the interpretation of these algorithms in the context of pure
unsupervised learning, as devices to capture the salient structure of the
input data distribution. Whereas the answer is clear for RBMs, it is less
obvious for regularized auto-encoders. Do they completely characterize the
input distribution or only some aspect of it? For example, clustering
algorithms such as k-means only capture the modes of the distribution,
while manifold learning algorithms characterize the low-dimensional regions
where the density concentrates.

Some of the first ideas about the probabilistic interpretation of
auto-encoders were proposed by~\citet{ranzato-08}: they were viewed as
approximating an energy function through the reconstruction error, i.e.,
being trained to have low reconstruction error at the training examples and
high reconstruction error elsewhere (through the regularizer, e.g.,
sparsity or otherwise, which prevents the auto-encoder from learning the
identity function).  An important breakthrough then came, yielding a first
formal probabilistic interpretation of regularized auto-encoders as models
of the input distribution, with the work
of~\citet{Vincent-NC-2011-small}. This work showed that some denoising
auto-encoders (DAEs) correspond to a Gaussian RBM and that minimizing the
denoising reconstruction error (as a squared error) estimates the energy
function through a regularized form of score matching, with the
regularization disappearing as the amount of corruption noise goes to 0,
and then converging to the same solution as score
matching~\citep{Hyvarinen-2005}. This connection and its generalization
to other energy functions, giving rise to the general denoising
score matching training criterion, is discussed in several other
papers~\citep{Kingma+LeCun-2010,Swersky-ICML2011,Alain+Bengio-ICLR2013}.

Another breakthrough has been the development of an empirically successful
sampling algorithm for contractive auto-encoders~\citep{Rifai-icml2012-small},
which basically involves composing encoding, decoding, and noise addition
steps. This algorithm is motivated by the observation that the Jacobian
matrix (of derivatives) of the encoding function provides an estimator of a
local Gaussian approximation of the density, i.e., the leading singular
vectors of that matrix span the tangent plane of the manifold near which
the data density concentrates.  However, a formal justification for this
algorithm remains an open problem.

The last step in this development~\citep{Alain+Bengio-ICLR2013} generalized
the result from~\citet{Vincent-NC-2011-small} by showing that when a
DAE (or a contractive auto-encoder with the contraction
on the whole encode/decode reconstruction function) is trained with small
Gaussian corruption and squared error loss, it estimates the score
(derivative of the log-density) of the underlying data-generating
distribution, which is proportional to the difference between
reconstruction and input. This result does not depend on the
parametrization of the auto-encoder, but suffers from the following
limitations: it applies to one kind of corruption (Gaussian), only to
continuous-valued inputs, only for one kind of loss (squared error), and it
becomes valid only in the limit of small noise (even though in practice,
best results are obtained with large noise levels, comparable to the range
of the input). 

What we propose here is a different probabilistic interpretation of
DAEs, which is valid for any data type, any corruption
process (so long as it has broad enough support), and any reconstruction
loss (so long as we can view it as a log-likelihood).

The basic idea is that if we corrupt observed random variable $X$ into
$\tilde{X}$ using conditional distribution ${\cal C}(\tilde{X}|X)$, we are
really training the DAE to estimate the reverse
conditional $P(X | \tilde{X})$.
Combining this estimator with the known ${\cal
  C}(\tilde{X}|X)$, we show that we can recover a consistent estimator of
$P(X)$ through a Markov chain that alternates between sampling from $P(X |
\tilde{X})$ and sampling from ${\cal C}(\tilde{X}|X)$, i.e., encode/decode,
sample from the reconstruction distribution model $P(X |\tilde{X})$, 
apply the stochastic corruption procedure  ${\cal C}(\tilde{X}|X)$, and iterate. 

This theoretical result is validated through experiments on artificial data
in a non-parametric setting and experiments on real data in a parametric
setting (with neural net DAEs). We find that we can improve the
sampling behavior by using the model itself to define the corruption
process, yielding a training procedure that has some surface similarity to
the contrastive divergence algorithm~\citep{Hinton99-small,Hinton06}.

\begin{algorithm}[ht]
\caption{{\sc The generalized denoising auto-encoder training algorithm}
\sl requires a training set or training distribution $\cal D$ of examples $X$,
a given corruption process ${\cal C}(\tilde{X}|X)$ from which
one can sample, and with which one trains a conditional distribution $P_\theta(X|\tilde{X})$
from which one can sample.}
\begin{algorithmic} \label{alg:DAE}
\REPEAT
  \STATE $\bullet$ sample training example $X \sim {\cal D}$
  \STATE $\bullet$ sample corrupted input $\tilde{X} \sim {\cal C}(\tilde{X}|X)$
  \STATE $\bullet$ use $(X,\tilde{X})$ as an additional training example towards
  minimizing the expected value of $- \log P_\theta(X | \tilde{X})$, e.g.,
  by a gradient step with respect to $\theta$.
\UNTIL convergence of training (e.g., as measured by 
early stopping on out-of-sample negative log-likelihood)
\end{algorithmic}
\vs{1}
\end{algorithm}

\vs{3}
\section{Generalizing Denoising Auto-Encoders}
\vs{2}
\subsection{Definition and Training}
\vs{2}
Let ${\cal P}(X)$ be the data-generating distribution
over observed random variable $X$. Let ${\cal C}$ be a given corruption process that 
stochastically
maps an $X$ to a $\tilde{X}$ through conditional distribution ${\cal C}(\tilde{X}|X)$.
The training data for the generalized denoising auto-encoder is a set
of pairs $(X,\tilde{X})$ with $X \sim {\cal P}(X)$ and $\tilde{X} \sim {\cal C}(\tilde{X}|X)$.
The DAE is trained to predict $X$ given $\tilde{X}$ through
a learned conditional distribution $P_\theta(X|\tilde{X})$, by choosing this conditional
distribution within some family of distributions indexed by $\theta$, not necessarily a neural net.
The training procedure for the DAE can generally be formulated
as learning to predict $X$ given $\tilde{X}$ by possibly regularized
maximum likelihood, i.e., the generalization performance that this 
training criterion attempts to minimize is
\vs{1}
\begin{equation}
 \label{eq:true-L}
  {\cal L}(\theta) = - E[ \log P_\theta(X|\tilde{X}) ]
\end{equation}
\vs{2}
where 
the expectation is taken over the joint data-generating distribution
\vs{1}
\begin{equation}
\label{eq:joint}
 {\cal P}(X,\tilde{X}) = {\cal P}(X) {\cal C}(\tilde{X}|X).
\end{equation}
\vs{10}
\subsection{Sampling}
\vs{2}
We define the following pseudo-Gibbs Markov chain associated with $P_\theta$:
\vs{2}
\begin{align}
  X_t &\sim P_\theta(X|\tilde{X}_{t-1}) \nonumber\\
  \tilde{X}_t &\sim {\cal C}(\tilde{X}|X_t)
\end{align}
which can be initialized from an arbitrary choice $X_0$.
This is the process by which we are going to generate samples $X_t$
according to the model implicitly learned by choosing $\theta$.
We define $T(X_t | X_{t-1})$ the transition operator that defines
a conditional distribution for $X_t$ given $X_{t-1}$,
independently of $t$, so that the
sequence of $X_t$'s forms a homogeneous Markov chain.
If the asymptotic marginal distribution of the $X_t$'s exists,
we call this distribution $\pi(X)$, and we show below that
it consistently estimates ${\cal P}(X)$. 

Note that the above chain is not a proper Gibbs chain in general because
there is no guarantee that $P_\theta(X|\tilde{X}_{t-1})$ and ${\cal
  C}(\tilde{X}|X_t)$ are consistent with a unique joint distribution. In
that respect, the situation is similar to the sampling procedure for
dependency networks~\citep{HeckermanD2000}, in that the pairs
$(X_t,\tilde{X}_{t-1})$ are not guaranteed to have the same asymptotic
distribution as the pairs $(X_t,\tilde{X}_t)$ as $t \rightarrow\infty$.  As
a follow-up to the results in the next section, it is shown
in~\citet{Bengio+Laufer-arxiv-2013} that dependency networks can be cast
into the same framework (which is that of Generative Stochastic Networks),
and that if the Markov chain is ergodic, then its stationary distribution
will define a joint distribution between the random variables (here that
would be $X$ and $\tilde{X}$), even if the conditionals are not consistent
with it. 

\vs2
\subsection{Consistency}
\vs2

Normally we only have access to a finite number $n$ of training examples
but as $n\rightarrow \infty$, the empirical training distribution approaches the 
data-generating distribution. To compensate for the finite training set,
we generally introduce a (possibly data-dependent) regularizer $\Omega$
and the actual training criterion is a sum over $n$ training examples $(X,\tilde{X})$,
\begin{equation}
  {\cal L}_n(\theta) = \frac{1}{n} \sum_{X \sim {\cal P}(X),\tilde{X} \sim {\cal C}(\tilde{X}|X)}
    \lambda_n \Omega(\theta,X,\tilde{X}) - \log P_\theta(X | \tilde{X})
\end{equation}
where we allow the regularization coefficient $\lambda_n$ to be chosen according to the
number of training examples $n$, with $\lambda_n \rightarrow 0$ as $n\rightarrow \infty$.
With $\lambda_n \rightarrow 0$ we get that
${\cal L}_n \rightarrow {\cal L}$ (i.e. converges to generalization error, Eq.~\ref{eq:true-L}),
so consistent estimators of ${\cal P}(X|\tilde{X})$ stay consistent.
We define $\theta_n$ to be the minimizer of ${\cal L}_n(\theta)$ when given $n$ training examples.


We define $T_n$ to be the transition operator 
$T_n(X_t | X_{t-1}) = \int P_{\theta_n}(X_t | \tilde{X}) {\cal C}(\tilde{X}|X_{t-1}) d\tilde{X}$
associated with $\theta_n$ (the parameter
obtained by minimizing the training criterion with $n$ examples), and
define $\pi_n$ to be the asymptotic distribution of the Markov chain generated
by $T_n$ (if it exists). We also define $T$ be the operator of the Markov chain associated
with the learned model as $n \rightarrow \infty$.

\begin{theorem}
\label{thm:consistency}
If $P_{\theta_n}(X | \tilde{X})$ is a consistent estimator of the true conditional
distribution ${\cal P}(X | \tilde{X})$ and $T_n$ defines an 
ergodic Markov chain, then
as the number of examples $n\rightarrow \infty$, the asymptotic distribution $\pi_n(X)$ of the generated
samples converges to the data-generating distribution ${\cal P}(X)$. 
\vs2
\end{theorem}
\begin{proof}
If $T_n$ is ergodic, then the Markov chain converges to a $\pi_n$. 
Based on our definition of the ``true'' joint (Eq.~\ref{eq:joint}),
one obtains a conditional ${\cal P}(X|\tilde{X}) \propto {\cal P}(X){\cal C}(\tilde{X}|X)$.
This conditional, along with ${\cal P}(\tilde{X}|X)={\cal C}(\tilde{X}|X)$ 
can be used to define a proper Gibbs chain where one alternatively samples
from ${\cal P}(\tilde{X}|X)$ and from ${\cal P}(X|\tilde{X})$.
Let ${\cal T}$ be the corresponding ``true'' transition operator, which
maps the $t$-th sample $X$ to the $t+1$-th in that chain.
That is, ${\cal T}(X_t | X_{t-1}) = \int {\cal P}(X_t | \tilde{X}) {\cal C}(\tilde{X}|X_{t-1}) d\tilde{X}$.
${\cal T}$ produces ${\cal P}(X)$ as asymptotic marginal distribution over $X$
(as we consider more samples from the chain) simply because ${\cal P}(X)$
is the marginal distribution of the joint ${\cal P}(X){\cal C}(\tilde{X}|X)$
to which the chain converges.
By hypothesis we have that $P_{\theta_n}(X | \tilde{X})\rightarrow {\cal P}(X | \tilde{X})$
as $n\rightarrow \infty$. Note that $T_n$ is defined exactly as ${\cal T}$ but
with ${\cal P}(X_t|\tilde{X})$ replaced by $P_{\theta_n}(X | \tilde{X})$. Hence
$T_n \rightarrow {\cal T}$ as $n\rightarrow\infty$. 

Now let us convert the convergence of $T_n$ to ${\cal T}$ into the convergence
of $\pi_n(X)$ to ${\cal P}(X)$. We will exploit the fact that for the 2-norm, matrix $M$
and unit vector $v$,
\mbox{$||M v||_2 \leq \sup_{||x||_2=1} ||M x||_2 = ||M||_2$}.
Consider $M = {\cal T}-T_n$ and $v$ the principal eigenvector of ${\cal T}$,
which, by the Perron-Frobenius theorem, corresponds to the asymptotic distribution ${\cal P}(X)$.
Since $T_n \rightarrow {\cal T}$, $||{\cal T}-T_n||_2 \rightarrow 0$.
Hence $||({\cal T} - T_n) v||_2 \leq ||{\cal T}-T_n||_2 \rightarrow 0$, which 
implies that $T_n v \rightarrow {\cal T} v = v$, where the last equality
comes from the Perron-Frobenius theorem (the leading eigenvalue is 1).
Since $T_n v \rightarrow v$, it implies that $v$ becomes the leading eigenvector
of $T_n$, i.e., the asymptotic distribution of the Markov chain, $\pi_n(X)$ converges
to the true data-generating distribution, ${\cal P}(X)$, as $n \rightarrow \infty$.
\vs2
\end{proof}

Hence the asymptotic sampling distribution associated with the Markov chain defined by
$T_n$ (i.e., the model)
implicitly defines the distribution $\pi_n(X)$ learned by the DAE
over the observed variable $X$. Furthermore, that estimator of ${\cal P}(X)$ is consistent
so long as our (regularized) maximum likelihood estimator of the conditional $P_\theta(X|\tilde{X})$
is also consistent. We now provide sufficient conditions for the ergodicity of the chain operator
(i.e. to apply theorem~\ref{thm:consistency}).
\begin{corollary}
\label{cor:ergo}
{\bf If} $P_\theta(X | \tilde{X})$ is a consistent estimator of the true conditional
distribution ${\cal P}(X | \tilde{X})$, {\bf and} both the data-generating distribution and
denoising model are contained in and non-zero in
a finite-volume region $V$ (i.e., $\forall \tilde{X}$, $\forall X\notin V,\; {\cal P}(X)=0, P_\theta(X|\tilde{X})=0$),
{\bf and} $\forall \tilde{X}$, $\forall X\in V,\;$ \mbox{${\cal P}(X)>0$}, $P_\theta(X|\tilde{X})>0,  {\cal C}(\tilde{X}|X)>0$
{\bf and} these statements remain
true in the limit of $n\rightarrow \infty$, {\bf then}
the asymptotic distribution $\pi_n(X)$ of the generated
samples converges to the data-generating distribution ${\cal P}(X)$.
\vs2
\end{corollary}
\begin{proof}
To obtain the existence of a stationary distribution, it is sufficient to have irreducibility
(every value reachable from every other value), aperiodicity (no cycle where
only paths through the cycle allow to return to some value), and 
recurrence (probability 1 of returning eventually). These conditions
can be generalized to the continuous case, where we obtain ergodic Harris chains
rather than ergodic Markov chains.
If $P_\theta(X|\tilde{X})>0$ and ${\cal C}(\tilde{X}|X)>0$
(for $X \in V$), then $T_n(X_t | X_{t-1})>0$ as well,
because 
\[
  T(X_t | X_{t-1}) = \int P_\theta(X_t | \tilde{X}) {\cal C}(\tilde{X}|X_{t-1}) d\tilde{X}
\]
This positivity of the transition operator guarantees that one can jump from any point in
$V$ to any other point in one step, thus yielding {\em irreducibility} and {\em
  aperiodicity}. To obtain {\em recurrence} (preventing the chain from diverging to
infinity), we rely on the assumption that the domain $V$ is bounded. Note that although
$T_n(X_t | X_{t-1})>0$ could be true for any finite $n$, we need this condition to hold
for $n \rightarrow \infty$ as well, to obtain the consistency result of
theorem~\ref{thm:consistency}. By assuming this positivity (Boltzmann
distribution) holds for the data-generating distribution, we make sure that $\pi_n$ does
not converge to a distribution which puts 0's anywhere in $V$. Having satisfied all the
conditions for the existence of a stationary distribution for $T_n$ as $n\rightarrow
\infty$, we can apply theorem~\ref{thm:consistency} and obtain its conclusion.  \vs2
\end{proof}

Note how these conditions take care of the various troubling cases one
could think of. We avoid the case where there is no corruption
(which would yield a wrong estimation, with the DAE
simply learning a dirac probability its input).
Second, we avoid the case where the chain wanders to
infinity by assuming a finite volume where the model and data live, a real
concern in the continuous case. If it became a real issue, we
could perform rejection sampling to make sure that $P(X|\tilde{X})$
produces $X \in V$.

\vs{3}
\subsection{Locality of the Corruption and Energy Function}
\vs{2}
If we believe that $P(X | \tilde{X})$ is well estimated 
for all $(X,\tilde{X})$ pairs, i.e., that it is approximately consistent 
with ${\cal C}(\tilde{X} | X)$, then we get as many estimators
of the energy function as we want, by picking a particular
value of $\tilde{X}$. 

Let us define the notation $P(\cdot)$ to denote the probability
of the joint, marginals or conditionals over the pairs $(X_t,\tilde{X}_{t-1})$
that are produced by the model's Markov chain $T$ as $t \rightarrow \infty$.
So $P(X)=\pi(X)$ is the asymptotic
distribution of the Markov chain $T$, and $P(\tilde{X})$ the
marginal over the $\tilde{X}$'s in that chain.
The above assumption means that 
$P(\tilde{X}_{t-1}|X_t) \approx {\cal C}(\tilde{X}_{t-1}|X_t)$
(which is not guaranteed in general, but only asymptotically 
as $P$ approaches the true ${\cal P}$).
Then, by Bayes rule,
$P(X)= \frac{P(X|\tilde{X}) P(\tilde{X})}{P(\tilde{X}|X)}
       \approx \frac{P(X|\tilde{X}) P(\tilde{X})}{C(\tilde{X}|X)}
        \propto \frac{P(X|\tilde{X})}{C(\tilde{X}|X)}$
so that we can get an estimated energy function from any given choice of $\tilde{X}$
through  ${\rm energy}(X) \approx -\log P(X|\tilde{X}) + \log C(\tilde{X}|X)$.
where one should note that the intractable {\em partition function depends on 
the chosen value of $\tilde{X}$}.

How much can we trust that estimator and how should $\tilde{X}$ be chosen?
First note that $P(X|\tilde{X})$ has only been trained for pairs $(X,\tilde{X})$
for which $\tilde{X}$ is relatively close to $X$ (assuming that the corruption
is indeed changing $X$ generally into some neighborhood). Hence, although in theory
(with infinite amount of data and capacity) the above estimator should be good,
in practice it might be poor when $X$ is far from $\tilde{X}$. So if we pick
a particular $\tilde{X}$ the estimated energy might be good for $X$ in the
neighborhood of $\tilde{X}$ but poor elsewhere. What we could do though,
is use a different approximate energy function in different regions of
the input space. Hence the above estimator gives us a way to compare
the probabilities of nearby points $X_1$ and $X_2$ (through their difference
in energy), picking for example a midpoint $\tilde{X}=\frac{1}{2}(X_1+X_2)$.
One could also imagine that if $X_1$ and $X_N$ are far apart, we could
chart a path between $X_1$ and $X_N$ with intermediate points $X_k$
and use an estimator of the relative energies between the neighbors $X_k,X_{k+1}$,
add them up, and obtain an estimator of the relative energy
between $X_1$ and $X_N$.

\begin{figure}[ht]
\vs{4}
\centerline{\includegraphics[width=0.45\columnwidth]{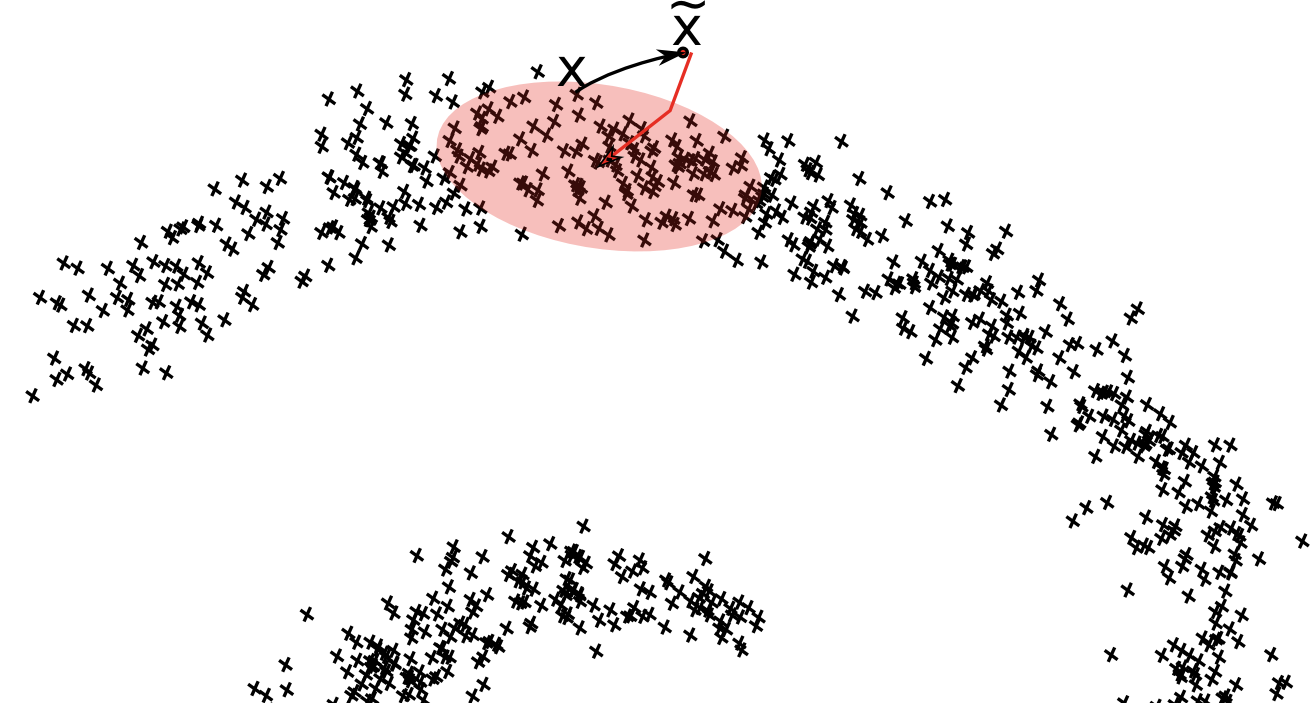}}
\vs{3}
\caption{\sl Although ${\cal P}(X)$ may be complex and multi-modal,
${\cal P}(X|\tilde{X})$ is often simple and approximately unimodal (e.g., multivariate
Gaussian, pink oval)
for most values of $\tilde{X}$ when ${\cal C}(\tilde{X}|X)$ is a local
corruption. ${\cal P}(X)$ can be seen as an infinite mixture
of these local distributions (weighted by ${\cal P}(\tilde{X})$).}
\label{fig:local-estimator}
\vs{4}
\end{figure}

This brings up an interesting point. If we could always obtain a good
estimator $P(X|\tilde{X})$ for any $\tilde{X}$, we could just train the
model with ${\cal C}(\tilde{X}|X)={\cal C}(\tilde{X})$, i.e., with an
unconditional noise process that ignores $X$. 
In that case, the estimator $P(X | \tilde{X})$
would directly equal $P(X)$ since $\tilde{X}$ and $X$ are actually
sampled independently in its ``denoising'' 
training data. We would have gained nothing over just
training any probabilistic model just directly modeling the observed
$X$'s. The gain we expect from using the denoising framework
is that if $\tilde{X}$ is a local perturbation of $X$, then the true ${\cal
  P}(X|\tilde{X})$ can be well approximated by a much simpler
distribution than ${\cal P}(X)$. See
Figure~\ref{fig:local-estimator} for a visual explanation: in the limit of
very small perturbations, one could even assume that ${\cal
  P}(X|\tilde{X})$ can be well approximated by a simple unimodal
distribution such as the Gaussian (for continuous data) or factorized
binomial (for discrete binary data) commonly used in DAEs
as the reconstruction probability function (conditioned on
$\tilde{X}$). This idea is already behind the non-local manifold Parzen
windows~\citep{Bengio-Larochelle-NLMP-NIPS-2006} and non-local manifold
tangent learning~\citep{Bengio+Monperrus+Larochelle-2006-small} algorithms:
the local density around a point $\tilde{X}$ can be approximated by a
multivariate Gaussian whose covariance matrix has leading eigenvectors that
{\em span the local tangent of the manifold} near which the data
concentrates (if it does). The idea of a locally Gaussian approximation of
a density with a manifold structure is also exploited in the more recent work on
the contractive auto-encoder~\citep{Rifai+al-2011-small} and associated
sampling procedures~\citep{Rifai-icml2012-small}. Finally, strong theoretical
evidence in favor of that idea comes from the result from \citet{Alain+Bengio-ICLR2013}:
{\em when the amount of corruption noise converges to 0 and the input variables
have a smooth continuous density, then a unimodal Gaussian reconstruction
density suffices to fully capture the joint distribution.}

Hence, although $P(X|\tilde{X})$ encapsulates all information about
${\cal P}(X)$ (assuming ${\cal C}$ given), 
it will generally have far fewer non-negligible modes, making easier to approximate it. 
This can be seen analytically
by considering the case where ${\cal P}(X)$ is a mixture of many Gaussians
and the corruption is a local Gaussian: $P(X|\tilde{X})$ remains
a Gaussian mixture, but one for which most of the modes have become negligible~\citep{Alain+Bengio-ICLR2013}.
We return to this in Section~\ref{sec:spurious}, suggesting that in order
to avoid spurious modes, it is better to have non-infinitesimal
corruption, allowing faster mixing and successful burn-in not pulled by
spurious modes far from the data.

\vs{3}
\section{Reducing the Spurious Modes with Walkback Training}
\vs{2}
\label{sec:spurious}

Sampling in high-dimensional spaces (like in experiments
below) using a simple local corruption process (such as Gaussian or
salt-and-pepper noise) suggests that
if the corruption is too local, the DAE's behavior far
from the training examples can create spurious modes in the regions
insufficiently visited during training. More training iterations or 
increasing the amount of corruption noise
helps to substantially alleviate that problem, but we discovered an even
bigger boost by {\em training
the DAE Markov chain to walk back towards the training examples}
(see Figure \ref{fig:walkback_into_drain}).
\begin{SCfigure}
\centering
\includegraphics[width=0.49\textwidth]{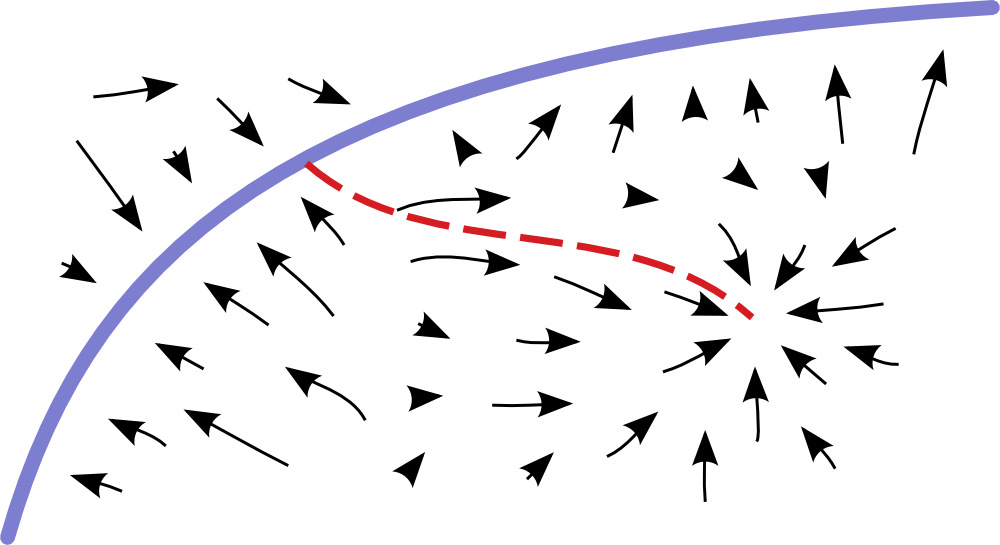}
\caption{
Walkback samples get attracted by spurious modes and contribute to removing them.
Segment of data manifold in violet and example walkback path in red dotted line,
starting on the manifold and going towards a spurious attractor.
The vector field represents expected moves of the chain,
for a unimodal $P(X| \tilde{X})$, with arrows from $\tilde{X}$ to $X$.
}
\label{fig:walkback_into_drain}
\vspace{-0.5cm}
\end{SCfigure}
We exploit
knowledge of the currently learned model $P(X|\tilde{X})$ to define the
corruption, so as to pick values of $\tilde{X}$ that would be
obtained by following the generative chain: wherever
the model would go if we sampled using the generative Markov chain starting
at a training example $X$,
we consider to be a kind of ``negative example'' $\tilde{X}$
from which the
auto-encoder should move away (and towards $X$).
The spirit of this procedure is thus very similar
to the CD-$k$ (Contrastive Divergence with $k$ MCMC steps) procedure
proposed to train RBMs~\citep{Hinton99-small,Hinton06}.

More precisely, the modified corruption process ${\cal \tilde{C}}$ we
propose is the following, based on the original corruption process $\cal
C$. We use it in a version of the training algorithm called {\bf walkback},
where we replace the corruption process ${\cal C}$ of
Algorithm~\ref{alg:DAE} by the walkback process ${\cal \tilde{C}}$ of
Algorithm~\ref{alg:walkback}. This also provides extra training examples (taking
advantage of the $\tilde{X}$ samples generated along the walk away from
$X$). It is called {\bf walkback} because it forces the DAE
to learn to walk back from the random walk it generates,
towards the $X$'s in the training set.

\begin{minipage}{\textwidth}
\rule{\linewidth}{0.4mm}
\vs{7}
\captionof{algorithm}{
{\sc The walkback algorithm} \sl is based on the walkback corruption
  process ${\cal \tilde{C}}(\tilde{X}|X)$, defined below in terms of a generic
  original corruption process ${\cal C}(\tilde{X}|X)$ and the current
  model's reconstruction conditional distribution $P(X|\tilde{X})$.  For
  each training example $X$, it provides a sequence of additional training examples
  $(X,\tilde{X}^*)$ for the DAE. It has a hyper-parameter
  that is a geometric distribution parameter $0<p<1$ controlling the
  length of these walks away from $X$, with
  $p=0.5$ by default. Training by Algorithm~\ref{alg:DAE} is the same, 
  but using all $\tilde{X}^*$ in the returned list $L$
  to form the pairs $(X,\tilde{X}^*)$ as
  training examples instead of just $(X,\tilde{X})$.  }
\vs{2}
\rule{\linewidth}{0.2mm}
\vs{2}
\begin{algorithmic}[1] \label{alg:walkback}
\vs{2}
\STATE $X^* \leftarrow X$, $L \leftarrow [\;]$
\STATE Sample $\tilde{X}^* \sim {\cal C}(\tilde{X}|X^*)$
\STATE Sample $u \sim {\rm Uniform}(0,1)$
\IF{$u > p$} 
  \STATE Append $\tilde{X}^*$ to $L$ and {\bf return} $L$
\ENDIF
\STATE If during training, append $\tilde{X}^*$ to $L$, so  $(X,\tilde{X}^*)$ will be an additional training example.
\STATE Sample $X^* \sim P(X | \tilde{X}^*)$
\STATE {\bf goto} 2.
\vs{2}
\end{algorithmic}
\rule{\linewidth}{0.2mm}
\vs{1.25}
\end{minipage}
\vs{1}

\begin{proposition}
\label{p-walkback}
Let $P(X)$ be the implicitly defined asymptotic distribution of the Markov chain alternating
sampling from $P(X|\tilde{X})$ and ${\cal C}(\tilde{X}|X)$, where ${\cal C}$
is the original local corruption process.
Under the assumptions of corollary~\ref{cor:ergo},
minimizing the training criterion 
in walkback training algorithm for generalized DAEs
(combining Algorithms~\ref{alg:DAE} and \ref{alg:walkback}) produces
a $P(X)$ that is a consistent estimator 
of the data generating distribution ${\cal P}(X)$.
\vs{1}
\end{proposition}
\vs{3}
\begin{proof}
Consider that during training, we produce a sequence of estimators
$P_k(X|\tilde{X})$ where $P_k$ corresponds to the $k$-th training iteration
(modifying the parameters after each iteration). With the walkback
algorithm, $P_{k-1}$ is used to obtain the corrupted samples $\tilde{X}$
from which the next model $P_k$ is produced. If training
converges, $P_k \approx P_{k+1} = P$ and we can then consider the whole corruption
process ${\cal \tilde{C}}$ fixed. By corollary~\ref{cor:ergo},
the Markov chain obtained by alternating samples from $P(X|\tilde{X})$
and samples from ${\cal \tilde{C}}(\tilde{X}|X)$ converges to
an asymptotic distribution $P(X)$ which estimates the underlying
data-generating distribution ${\cal P}(X)$. The walkback
corruption ${\cal \tilde{C}}(\tilde{X}|X)$ corresponds to a few
steps alternating sampling from ${\cal C}(\tilde{X}|X)$ (the fixed local
corruption)  and sampling from $P(X|\tilde{X})$. Hence the overall sequence
when using ${\cal \tilde{C}}$
can be seen as a Markov chain obtained
by alternatively sampling from ${\cal C}(\tilde{X}|X)$ and from 
$P(X|\tilde{X})$ just as it was when using merely ${\cal C}$. 
Hence, once the model is trained with
walkback, one can sample from it usingc orruption ${\cal C}(\tilde{X}|X)$.
\vs{3}
\end{proof}
A consequence is that {\em the walkback training
algorithm
estimates the same distribution as the original denoising algorithm}, but
may do it more efficiently (as we observe in the experiments),
by exploring the space of corruptions in a way that spends more time
where it most helps the model.

\vs{3}
\section{Experimental Validation}
\vs{2}

{\bf Non-parametric case.}
The mathematical results presented here
apply to any denoising training criterion where the reconstruction
loss can be interpreted as a negative log-likelihood. This
remains true whether or not the denoising machine $P(X|\tilde{X})$
is parametrized as the composition of an encoder and decoder.
This is also true of the asymptotic estimation results
in~\citet{Alain+Bengio-ICLR2013}.
We experimentally validate the above theorems in a case where the asymptotic limit
(of enough data and enough capacity) can be reached, i.e., in a
low-dimensional non-parametric setting.
Fig.~\ref{fig:histogram} shows the distribution
recovered by the Markov chain for {\bf discrete data}
with only 10 different values.  The conditional
$P(X|\tilde{X})$ was estimated by multinomial models and maximum likelihood (counting) from 5000
training examples. 5000 samples were generated from the chain to
estimate the asymptotic distribution $\pi_n(X)$.  For 
{\bf continuous data}, Figure~\ref{fig:histogram} also
shows the result of 5000 generated samples and 500 original training examples
with $X \in \R^{10}$, with scatter plots of pairs of
dimensions. The estimator is also non-parametric (Parzen density estimator
of $P(X|\tilde{X})$).
\begin{figure}[htb]
\centering
\vs{8}
\includegraphics[width=0.40\textwidth]{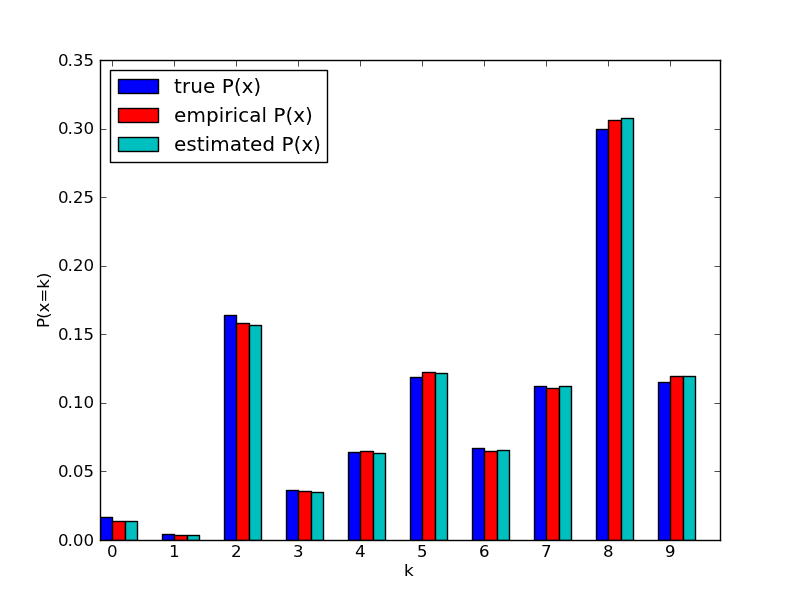}
\includegraphics[width=0.40\textwidth]{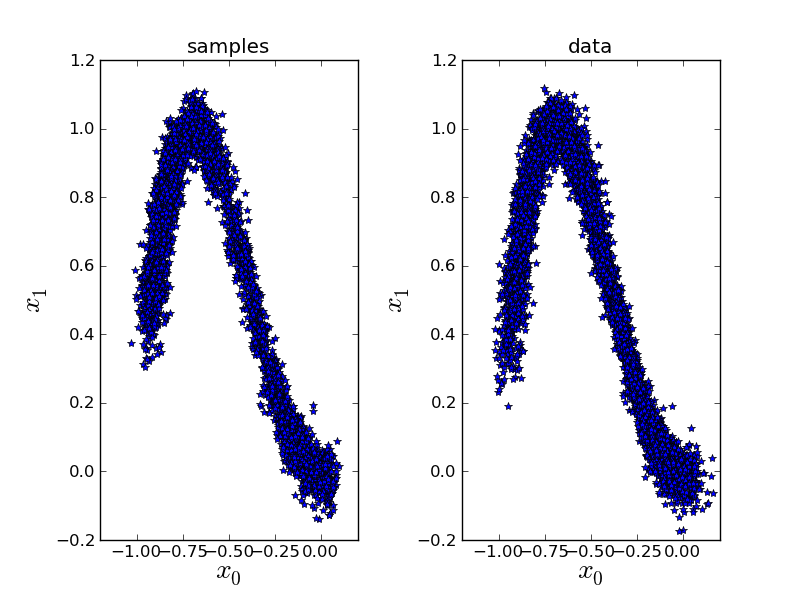}
\vs{2}
\includegraphics[width=0.40\textwidth]{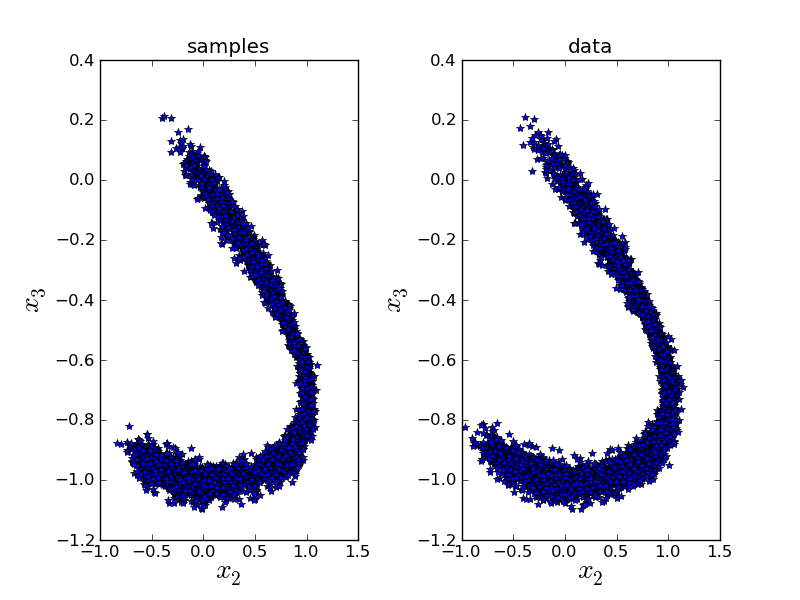}
\includegraphics[width=0.40\textwidth]{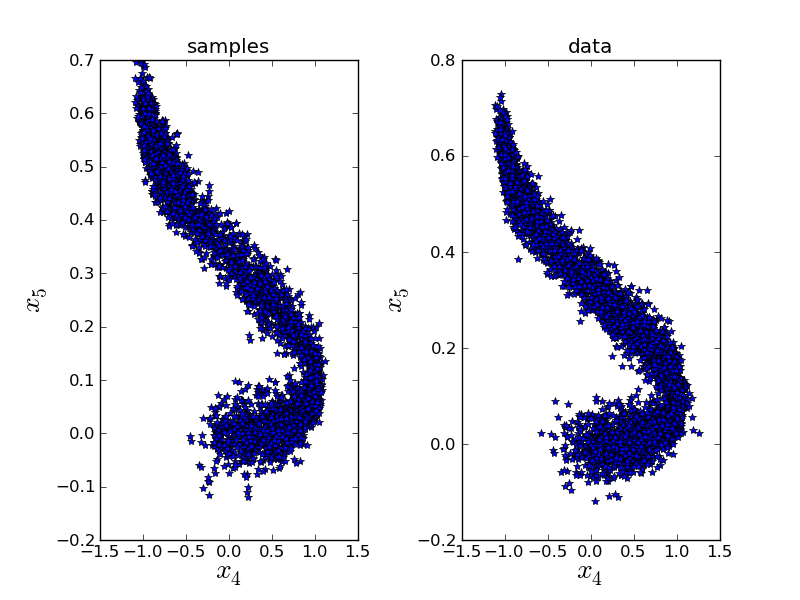}
\vs{1}
\caption{\sl Top left: histogram of a data-generating distribution (true, blue),
the empirical distribution (red), and the estimated distribution using
a denoising maximum likelihood estimator. Other figures: pairs of variables
(out of 10) showing the training samples and the model-generated samples.}
\label{fig:histogram}
\vs{4}
\end{figure}

{\bf MNIST digits.}  We trained a DAE on the
binarized MNIST data (thresholding at 0.5). 
A Theano\footnote{http://deeplearning.net/software/theano/}
 \citep{bergstra+al:2010-scipy-short} 
implementation
is available\footnote{git@github.com:yaoli/GSN.git}.
The 784-2000-784 auto-encoder is trained for 200 epochs with the 50000 training examples and salt-and-pepper noise
(probability 0.5 of corrupting each bit, setting it to 1 or 0 with
probability 0.5). It has 2000 tanh hidden units and is trained by minimizing cross-entropy loss,
i.e., maximum likelihood on a factorized Bernoulli reconstruction distribution.
With walkback training, a chain of 5
steps was used to generate 5 corrupted examples for each training
example. Figure~\ref{fig:MNIST} shows samples generated with and without walkback.
The quality of the samples was also estimated quantitatively by measuring
the log-likelihood of the test set under a non-parametric density
estimator $\hat{P}(x)={\rm mean}_{\tilde{X}} P(x|\tilde{X})$
constructed from 10000 consecutively generated samples
($\tilde{X}$ from the Markov chain). The expected value of $E[\hat{P}(x)]$
over the samples can be shown~\citep{Bengio+Yao-arxiv-2013} to be
a lower bound (i.e. conservative estimate) of the true (implicit) model density $P(x)$.
The test set log-likelihood bound
was not used to select among model architectures, but visual inspection of
samples generated did guide the preliminary search reported here.
Optimization hyper-parameters (learning rate, momentum, and
learning rate reduction schedule) were selected based on the 
training objective. We compare against a state-of-the-art RBM~\citep{NECO_cho_2013_enhanced}
with an AIS log-likelihood estimate of -64.1 (AIS estimates tend to be optimistic).
We also drew samples from the RBM and applied the same estimator (using the mean of the RBM's $P(x|h)$ with $h$
sampled from the Gibbs chain), and obtained a log-likelihood non-parametric bound of -233,
skipping 100 MCMC steps between samples (otherwise
numbers are very poor for the RBM, which does not mix at all).
The DAE log-likelihood bound
with and without walkback is respectively -116 and -142,
confirming visual inspection suggesting that
the walkback algorithm produces less spurious samples. However, the
RBM samples can be improved by a spatial blur. By tuning the amount of
blur (the spread of the Gaussian convolution), we obtained a bound of -112
for the RBM. Blurring did not help the auto-encoder. 

\begin{figure}[htb]
\centering
\vs{2}
\hspace*{-1mm}\includegraphics[width=0.5\textwidth]{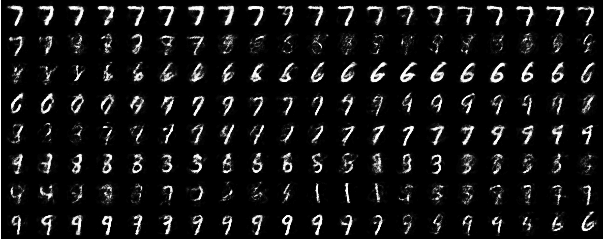} \hspace*{-1mm} \includegraphics[width=0.5\textwidth]{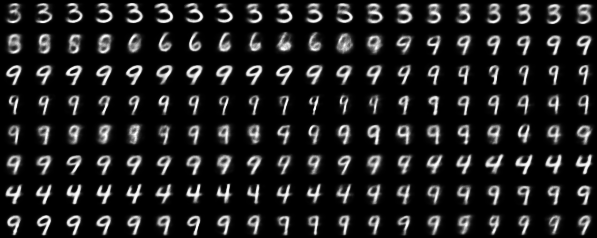}
\vs{5}
\caption{\sl Successive samples generated by Markov chain associated with
the trained DAEs
according to the plain sampling scheme (left) and walkback sampling scheme (right).
There are less ``spurious'' samples with the walkback algorithm.}
\label{fig:MNIST}
\vs{2}
\end{figure}

\vs{3}
\section{Conclusion and Future Work}
\vs{3}


We have proven that training a model to denoise is a way to implicitly
estimate the underlying data-generating process, and that a simple Markov
chain that alternates sampling from the denoising model and from the
corruption process converges to that estimator. This provides a means
for generating data from any DAE (if the corruption
is not degenerate, more precisely, if the above chain converges).
We have validated those results empirically, both in a non-parametric
setting and with real data. This study has
also suggested a variant of the training procedure,
{\em walkback training}, which seem to converge
faster to same the target distribution.

One of the insights arising out of the theoretical results presented here
is that in order to reach the asymptotic limit of fully capturing the data
distribution ${\cal P}(X)$, it may be necessary for the model's $P(X |
\tilde{X})$ to have the ability to represent multi-modal distributions over
$X$ (given $\tilde{X}$). 

\ifnipsfinal
\subsubsection*{Acknowledgments}
\vspace*{-1.5mm}

The authors would acknowledge input from
A. Courville, I. Goodfellow, R. Memisevic, K. Cho
as well as funding from NSERC, CIFAR (YB is a CIFAR Fellow),
and Canada Research Chairs.
\vspace*{-1.5mm}
\fi



\bibliography{strings,strings-shorter,ml,aigaion-shorter}
\bibliographystyle{natbib}

\end{document}